%% file: distribution_shift.tex
\documentclass{article}
\usepackage{times}

\usepackage{fancyhdr}
\usepackage{color}
\usepackage{natbib}
\usepackage{eso-pic} %
\usepackage{forloop}
\usepackage{float}

\usepackage{microtype}
\usepackage{graphicx}
\usepackage{subfigure}
\usepackage{booktabs} %
\usepackage[table,xcdraw]{xcolor}
\usepackage{hyperref}

\input{math_commands}

\usepackage[margin=1.28in]{geometry}

\title{A Theory of Label Propagation\\for Subpopulation Shift}
\author{Tianle Cai \thanks{Princeton University.} \thanks{Zhongguancun Haihua Institute for Frontier Information Technology} \thanks{Alphabetical order.}\and Ruiqi Gao$^{*\dagger\ddagger}$ \and Jason D.~Lee$^{*\ddagger}$ \and Qi Lei$^{*\ddagger}$}
\begin{document}
\maketitle
\begin{abstract}
One of the central problems in machine learning is domain adaptation. Unlike past theoretical work, we consider a new model for subpopulation shift in the input or representation space. In this work, we propose a provably effective framework for domain adaptation based on label propagation. In our analysis, we use a simple but realistic “expansion” assumption, proposed in \citet{wei2021theoretical}. Using a teacher classifier trained on the source domain, our algorithm not only propagates to the target domain but also improves upon the teacher. By leveraging existing generalization bounds, we also obtain end-to-end finite-sample guarantees on the entire algorithm. In addition, we extend our theoretical framework to a more general setting of source-to-target transfer based on a third unlabeled dataset, which can be easily applied in various learning scenarios. Inspired by our theory, we adapt consistency-based semi-supervised learning methods to domain adaptation settings and gain significant improvements.
\end{abstract}

\input{intro}

\input{related}
\input{label_prop_da}

\input{proof_sketch}
\input{label_prop_generalized}

\input{experiments}

\input{conclusion}
\section*{Acknowledgements}
JDL acknowledges support of the ARO under MURI Award W911NF-11-1-0303, the Sloan Research Fellowship, NSF CCF 2002272, and an ONR Young Investigator Award. QL is supported by NSF \#2030859 and the Computing Research Association for the CIFellows Project. We thank Prof. Yang Yuan for providing computational resources. We also thank Difan Zou for pointing out a mistake in the original proof of Lemma~\ref{lem_minority_2} which is now corrected in the revision.

\bibliography{ref}
\bibliographystyle{apalike}

\newpage

\appendix
\input{full_proof}

\input{proof_finite}
\input{exp_detail}

\input{other_related}

\end{document}

%% file: math_commands.tex
\usepackage{mathtools}
\usepackage{amssymb}
\usepackage{latexsym}
\usepackage{dsfont}
\usepackage{mathrsfs}
\usepackage{amssymb}
\usepackage{amsfonts}
\usepackage{bm}
\usepackage{xspace}
\usepackage{amsthm}

\newcommand*{\argmin}{\mathop{\mathrm{argmin}}}
\newcommand*{\argmax}{\mathop{\mathrm{argmax}}}

\newcommand{\Acal}{\mathcal{A}}
\newcommand{\Bcal}{\mathcal{B}}

\newcommand{\Ncal}{\mathcal{N}}

\newcommand{\Xcal}{\mathcal{X}}
\newcommand{\Ycal}{\mathcal{Y}}

\newcommand{\Pbb}{\mathbb{P}}

\newcommand{\Rbb}{\mathbb{R}}

\ifx\BlackBox\undefined
\newcommand{\BlackBox}{\rule{1.5ex}{1.5ex}}  %
\fi

\ifx\QED\undefined
\def\QED{~\rule[-1pt]{5pt}{5pt}\par\medskip}
\fi

\ifx\proof\undefined
\newenvironment{proof}{\par\noindent{\em Proof:\ }}{\hfill\BlackBox\\[.0mm]}
\fi

\ifx\theorem\undefined
\newtheorem{theorem}{Theorem}[section]
\fi

\ifx\example\undefined
\newtheorem{example}{Example}[section]
\fi

\ifx\property\undefined
\newtheorem{property}{Property}
\fi

\ifx\lemma\undefined
\newtheorem{lemma}{Lemma}[section]
\fi

\ifx\proposition\undefined
\newtheorem{proposition}{Proposition}[section]
\fi

\ifx\remark\undefined
\newtheorem{remark}{Remark}
\fi

\ifx\corollary\undefined
\newtheorem{corollary}{Corollary}[section]
\fi

\ifx\definition\undefined
\newtheorem{definition}{Definition}[section]
\fi

\ifx\conjecture\undefined
\newtheorem{conjecture}{Conjecture}[section]
\fi

\ifx\axiom\undefined
\newtheorem{axiom}[theorem]{Axiom}
\fi

\ifx\claim\undefined
\newtheorem{claim}[theorem]{Claim}
\fi

\ifx\assumption\undefined
\newtheorem{assumption}{Assumption}
\fi

\ifx\problem\undefined
\newtheorem{problem}{Problem}[section]
\fi

\ifx\fact\undefined
\newtheorem{fact}{Fact}
\fi

\newcommand{\rbr}[1]{\left(#1\right)}
\newcommand{\sbr}[1]{\left[#1\right]}
\newcommand{\cbr}[1]{\left\{#1\right\}}

\newcommand{\half}{\frac{1}{2}}

\newcommand{\supp}{\mathrm{supp}}

%% file: intro.tex
\section{Introduction}
\label{sec:intro}
The recent success of supervised deep learning is built upon two crucial cornerstones: That the training and test data are drawn from an \emph{identical distribution}, and that representative \emph{labeled} data are available for training. However, in real-world applications, labeled data drawn from the same distribution as test data are usually unavailable. Domain adaptation~\citep{quionero2009dataset,saenko2010adapting} suggests a way to overcome this challenge by transferring the knowledge of labeled data from a source domain to the target domain. 

Without further assumptions, the transferability of information is not possible. Existing theoretical works have investigated suitable assumptions that can provide learning guarantees. Many of the works are based on the \emph{covariate shift} assumption \citep{heckman1979sample,shimodaira2000improving}, which states that the conditional distribution of the labels (given the input $x$) is invariant across domains, i.e., $p_S\rbr{y|x} = p_T\rbr{y|x}$. Traditional approaches usually utilize this assumption by further assuming that the source domain covers the support of the target domain. In this setting, importance weighting \citep{shimodaira2000improving,cortes2010learning,cortes2015adaptation,zadrozny2004learning} can be used to transfer information from source to target with theoretical guarantees. However, the assumption of covered support rarely holds in practice.

In the seminal works of \citet{ben2010theory,ganin2016domain}, the authors introduced a theory that enables generalization to out-of-support samples via distribution matching. %
They showed that the risk on the target domain can be bounded by the sum of two terms a) the risk on the source domain plus a discrepancy between source and  target domains, and b) the optimal joint risk that a function in the hypothesis class can achieve. Inspired by this bound, numerous domain-adversarial algorithms aimed at matching the distribution of source and target domains in the feature space have been proposed \citep{ajakan2014domain,long2015learning,ganin2016domain}. These methods show encouraging empirical performance on transferring information from domains with different styles, e.g., from colorized photos to gray-scale photos. However, the theory of distribution matching can be violated since
only two terms in the bound are optimized in the algorithms while the other term can be arbitrary large \citep{zhao2019learning,wu2019domain,li2020rethinking}. In practice, forcing the representation distribution of two domains to match may also fail in some settings. As an example, \citet{li2020rethinking} gives empirical evidence of this failure on datasets with subpopulation shift. \citet{li2020rethinking} describes a classification task between vehicle and person; subpopulation shift happens when the source vehicle class contains $50\%$ car and $50\%$ motorcycle, while the target vehicle class contains $10\%$ car and $90\%$ motorcycle.

In real-world applications, subpopulation shift is pervasive, and often in a fine-grained manner. The source domain will inevitably fail to capture the diversity of the target domain, and models will encounter unseen subpopulations in the target domain, e.g., unexpected weather conditions for self-driving or different diagnostic setups in medical applications~\citep{santurkar2021breeds}. The lack of theoretical understanding of subpopulation shift motivates us to study the following question:
\begin{center}
	\textit{How to provably transfer from source to target domain\\under subpopulation shift using unlabeled data?}
\end{center}
\begin{figure*}[ht]
	\centering
	\includegraphics[width=0.9\textwidth]{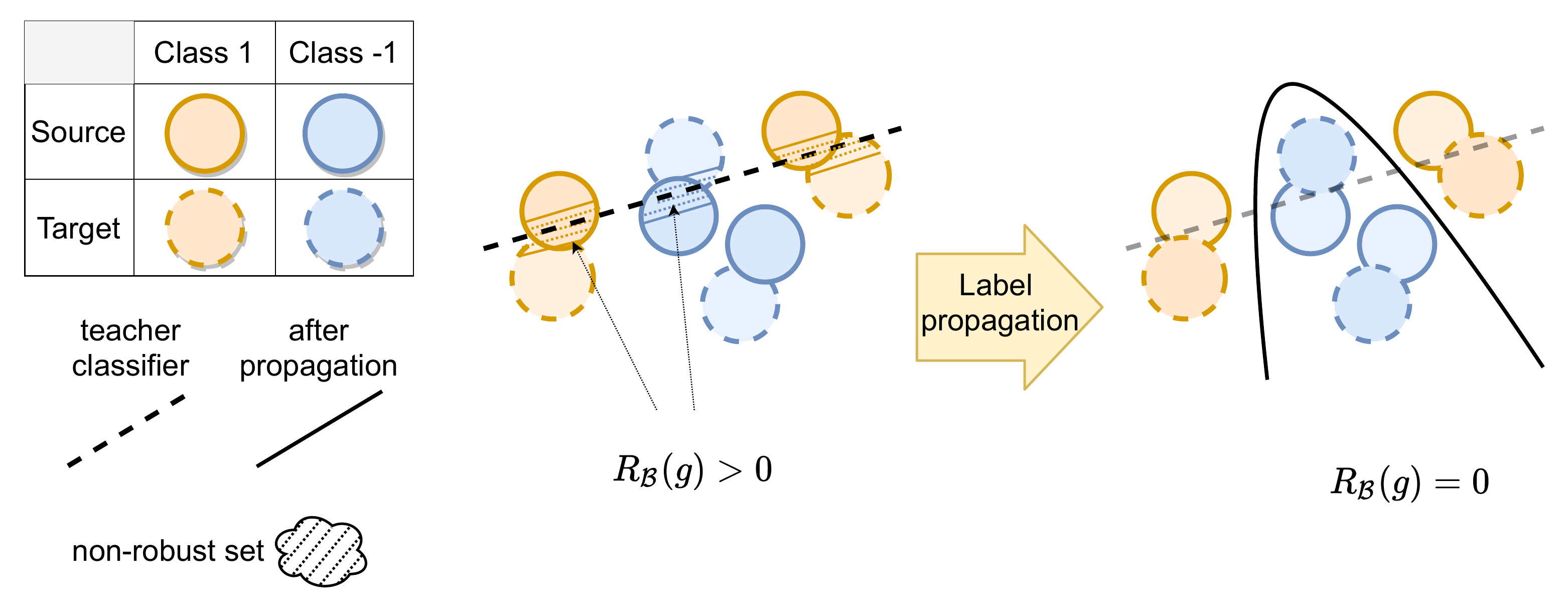}
	\caption{A toy illustration of our framework on label propogation on subpopulations, formalized in Section~\ref{sec:label_prop_da}. Although the formal definition (Assumption~\ref{asm_setting}) involves a neighborhood function $\Bcal(\cdot)$ and possibly a representation space, one can understand it by the above toy model: a set of $S_i$ and $T_i$ where each $S_i \cup T_i$ forms a regular connected component. The consistency loss $R_\Bcal(g)$ measures the amount of non-robust set of $g$, which contains points whose predictions by $g$ is inconsistent in a small neighborhood. Our main theorems (Theorem~\ref{thm_multiplicative} and \ref{thm_constant}) state that, starting from a teacher with information on the source data, consistency regularization (regularizing $R_\Bcal(g)$ on unlabeled data) can result in the propogation of label information, thereby obtaining a good classifier on the target domain, which may also improve upon the accuracy of the teacher on the source domain.}
	\label{fig:illustraion}
\end{figure*}
To address this question, we develop a general framework of domain adaptation where we have a supervision signal on the source domain (through a teacher classifier which has non-trivial performance on the source domain but is allowed to be entirely wrong on the target domain (See Assumption \ref{asm_setting}(a) and Figure~\ref{fig:illustraion}))
and unlabeled data on both source and target domains. The key of the analysis is to show that the supervision signal can be propagated to the unlabeled data. To do so, we partition data from both domains into some subpopulations and leverage a simple but realistic expansion assumption (Definition~\ref{def_expansion}) proposed in \citet{wei2021theoretical} on the subpopulations. We then prove that by minimizing a consistency regularization term~\citep{miyato2018virtual,shu2018dirt,xie2020unsupervised} on unlabeled data from both domains plus a 0-1 consistency loss with the supervision signal (i.e., the teacher classifier) on the source domain, the supervision signal from the subpopulations of the source domain can not only be propagated to the subpopulations of target domain but also refine the prediction on the source domain. In Theorem~\ref{thm_multiplicative} and~\ref{thm_constant}, we give bounds on the test performance on the target domain. Using off-the-shelf generalization bounds, we also obtain end-to-end finite-sample guarantees for neural networks in Section~\ref{sec:finite}. 

In Section~\ref{sec:label_prop_generalized}, we extend our theoretical framework to a more general setting with source-to-target transfer based on an additional unlabeled dataset. As long as the subpopulation components of the unlabeled dataset satisfy the expansion property and cover both the source and target subpopulation components, then one can provably propagate label information from source to target \emph{through} the unlabeled data distribution (Theorem~\ref{thm_multiplicative_2} and \ref{thm_constant_2}). As corollaries, we immediately obtain learning guarantees for both semi-supervised learning and unsupervised domain adaptation. The results can also be applied to various settings like domain generalization etc., see Figure~\ref{figure2}.

We implement the popular consistency-based semi-supervised learning algorithm FixMatch~\citep{sohn2020fixmatch} on the subpopulation shift task from BREEDS~\citep{santurkar2021breeds}, and compare it with popular distributional matching methods ~\citep{ganin2016domain, zhang2019bridging}. Results show that the consistency-based method outperforms distributional matching methods by over $8\%$, partially verifying our theory on the subpopulation shift problem. We also show that combining distributional matching methods and consistency-based algorithm can improve the performance upon distributional matching methods on classic unsupervised domain adaptation datasets such as Office-31~\citep{saenko2010adapting} and Office-Home~\citep{venkateswara2017deep}.

In summary, our contributions are: 1) We introduce a theoretical framework of learning under subpopulation shift through label propagation; 2) We provide accuracy guarantees on the target domain for a consistency-based algorithm using a fine-grained analysis under the expansion assumption~\citep{wei2021theoretical}; 3) We provide a generalized label propagation framework that easily includes several settings, e.g., semi-supervised learning, domain generalization, etc.

%% file: related.tex
\subsection{Related work} 
We review some more literature on domain adaptation, its variants, and consistency regularization, followed by discussions on the distinction of our contributions compared to~\citet{wei2021theoretical}. 

For the less challenging setting of covariate shift where the source domain covers the target domain's support, prior work regarding importance weighting focuses on estimations of the density ratio~\citep{lin2002support,zadrozny2004learning} through kernel mean matching~\citep{huang2006correcting,gretton2007kernel,zhang2013domain,shimodaira2000improving}, and some standard divergence minimization paradigms~\citep{sugiyama2008direct,sugiyama2012density,uehara2016generative,menon2016linking,kanamori2011f}. 
For out-of-support domain adaptation, recent work investigate approaches to match the source and target distribution in representation space~\citep{glorot2011domain,ajakan2014domain,long2015learning,ganin2016domain}. Practical methods involve designing domain adversarial objectives~\citep{tzeng2017adversarial,long2017conditional,hong2018conditional,he2019multi,xie2019multi,zhu2019adapting} or different types of discrepancy minimization~\citep{long2015learning,lee2019sliced,roy2019unsupervised,chen2020homm}.
Another line of work explore self-training or gradual domain adaptation~\citep{gopalan2011domain,gong2012geodesic,glorot2011domain,kumar2020understanding}. For instance, \citet{chen2020self} demonstrates that self-training tends to learn robust features in some specific probabilistic setting. %

Variants of domain adaptation have been extensively studied. For instance, weakly-supervised domain adaptation considers the case where the labels in the source domain can be noisy~\citep{shu2019transferable,liu2019butterfly}; multi-source domain adaptations adapts from multiple source domains~\citep{xu2018deep,zhao2018adversarial}; domain generalization also allows access to  multiple training environments, but seeks out-of-distribution generalization \emph{without} prior knowledge on the target domain~\citep{ghifary2015domain,li2018learning,arjovsky2019invariant,ye2021theoretical}.

The idea of consistency regularization has been used in many settings. \citet{miyato2018virtual,qiao2018deep,xie2020unsupervised} enforce consistency with respect to adversarial examples or data augmentations for semi-supervised learning. \citet{shu2019transferable} combines domain adversarial training with consistency regularization for unsupervised domain adaptation. Recent work on self-supervised learning also leverages the consistency between two aggressive data augmentations to learn meaningful features~\citep{chen2020simple,grill2020bootstrap,caron2020unsupervised}.

Most closely related to our work is \citet{wei2021theoretical}, which introduces a simple but realistic ``expansion'' assumption to analyze label propagation, which states that a low-probability subset of the data must expand to a neighborhood with larger probability relative to the subset. Under this assumption, the authors show learning guarantees for unsupervised learning and semi-supervised learning.

The focus of \citet{wei2021theoretical} is not on domain adaptation, though the theorems directly apply. This leads to several drawbacks that we now discuss. Notably, in the analysis of~\citet{wei2021theoretical} for unsupervised domain adaptation, the population test risk is bounded using the population risk of a pseudo-labeler on the \emph{target domain}.\footnote{In the new version of \citet{wei2021theoretical}, the authors proposed a refined result based on iterative training, which alleviates the error bound's dependency on the error of the target domain pseudo-labeler. Still, a mostly correct pretrained pseudo-labeler is required.} The pseudo-labeler is obtained via training with labeled data on the \emph{source domain}. For domain adaptation, we do not expect such a pseudo-labeler to be directly informative when applied to the target domain, especially when the distribution shift is severe. In contrast, our theorem does not rely on a good pseudo-labeler on the target domain. Instead, we prove that with only supervision on the source domain, the population risk on the target domain can converge to zero as the value of the consistency regularizer of the ground truth classifier decreases (Theorem~\ref{thm_multiplicative},~\ref{thm_constant}). In addition,~\citet{wei2021theoretical} assumes that the probability mass of \emph{each class} together satisfies the expansion assumption. However, each class may consist of several disjoint subpopulations. For instance, the dog class may have different breeds as its subpopulations. This setting differs from the concrete example of the Gaussian mixture model shown in~\citet{wei2021theoretical} where the data of each class concentrate following a Gaussian distribution. In this paper, we instead take a more realistic usage of the expansion assumption by assuming expansion property on the \emph{subpopulations} of each class (Assumption~\ref{asm_setting}). Behind this relaxation is a fine-grained analysis of the probability mass's expansion property, which may be of independent interest. 

%% file: label_prop_da.tex
\section{Label Propogation in Domain Adaptation}\label{sec:label_prop_da}
In this section, we consider label propagation for unsupervised domain adaptation. We assume the distributions' structure can be characterized by a specific subpopulation shift with the expansion property. In Section~\ref{sec:setting}, we introduce the setting, including the algorithm and assumptions. In Section~\ref{sec:main_theorem}, we present the main theorem on bounding the target error. In Section~\ref{sec:finite}, we provide an end-to-end guarantee on the generalization error of adapting a deep neural network to the target distribution with finite data. In Section~\ref{sec:proof_sketch} we provide a proof sketch for the theorems.

\subsection{Setting}\label{sec:setting}

We consider a multi-class classification problem $\Xcal \rightarrow \Ycal = \{1, \cdots, K\}$. Let $S$ and $T$ be the source and target distribution on $\Xcal$ respectively, and we wish to find a classifier $g: \Xcal \rightarrow \Ycal$ that performs well on $T$. Suppose we have a teacher classifier $g_{tc}$ on $S$. The teacher $g_{tc}$ can be obtained by training on the labeled data on $S$ (standard unsupervised domain adaptation), or by training on a small subset of labeled data on $S$, or by direct transferring from some other trained classifier, etc. In all, the teacher classifier represents all label information we know (and is allowed to have errors). Our goal is to transfer the information in $g_{tc}$ onto $T$ using only unlabeled data.

Our setting for subpopulation shift is formulated in the following assumption.

\begin{assumption}\label{asm_setting}
	Assume the source and target distributions have the following structure: $\supp(S) = \cup_{i=1}^{m}S_i$, $\supp(T) = \cup_{i=1}^{m}T_i$, where $S_i \cap S_j = T_i\cap T_j = S_i\cap T_j = \emptyset$ for all $i \neq j$. We assume the ground truth class $g^\ast(x)$ for $x\in S_i \cup T_i$ is consistent (constant), which is denoted as $y_i \in \{1, \cdots, K\}$. 
	We abuse the notation to let $S_i$, $T_i$ also denote the conditional distribution (probability measure) of $S, T$ on the set $S_i, T_i$ respectively. In addition, we make the following canonical assumptions:
	\begin{enumerate}
	\item The teacher classifier on $S_i$ is informative of the ground truth class $y_i$ by a margin $\gamma > 0$, that is,
	$$\begin{aligned}
	&\Pbb_{x\sim S_i}[g_{tc}(x)=y_i]\ge\Pbb_{x\sim S_i}[g_{tc}(x)=k] + \gamma, \\
	&\forall k \in \{1, \cdots, K\}\backslash\{y_i\}.
	\end{aligned}$$

	\item On each component, the ratio of the population under domain shift is upper-bounded by a constant $r$, i.e.
	$$\frac{\Pbb_{T}[T_i]}{\Pbb_S[S_i]} \le r, \forall i\in\{1, \cdots, m\}.$$
	\end{enumerate}
\end{assumption}

Following \citet{wei2021theoretical}, we make use of a consistency regularization method, i.e. we expect the predictions to be stable under a suitable set of input transformations $\Bcal(x) \subset \Xcal$. The regularizer of $g$ on the mixed probability measure $\frac{1}{2}(S+T)$ is defined as 
$$R_{\Bcal}(g) := \Pbb_{x\sim \frac{1}{2}\rbr{S + T}}[\exists x' \in \Bcal(x), \textrm{s.t. }g(x)\neq g(x')],$$
and a low regularizer value implies the labels are with high probability constant within $\Bcal(x)$.  Prior work on using consistency regularization for unlabeled self-training includes \citet{miyato2018virtual} where $\Bcal(\cdot)$ can be understood as a distance-based neighborhood set and \citet{adel2017unsupervised,xie2020unsupervised} where $\Bcal(\cdot)$ can be understood as the set of data augmentations. In general, $\Bcal(x)$ takes the form $\Bcal(x) = \{x': \exists A \in \Acal \text{ such that } d(x', \Acal(x)) \le r\}$\footnote{In this paper, consistency regularization, expansion property, and label propagation can also be understood as happening in a representation space, as long as $d(x, x') = \|h(x) -h(x')\|$ for some feature map $h$.} for a small number $r > 0$, some distance function $d$, and a class of data augmentation functions $\Acal$.

The set $\Bcal(x)$ is used in the following expansion property. First, for $x \in S_i\cup T_i$ ($i\in \{1, \cdots, m\}$), we define the neighborhood function $\Ncal$ as 
$$\Ncal(x) := (S_i \cup T_i) \cap \{x' | \Bcal(x)\cap \Bcal(x') \neq \emptyset\}$$

and the neighborhood of a set $A\in \Xcal $ as
$$\Ncal(A) := \cup_{x\in A \cap (\cup_{i=1}^{m}S_i\cup T_i)}\Ncal(x).$$
The expansion property on the mixed distribution $\half(S+T)$ is defined as follows:

\begin{definition}[Expansion \citep{wei2021theoretical}] \label{def_expansion}

~ \\\vspace{-.5cm}
	\begin{enumerate}
	\item (Multiplicative Expansion) We say $\half(S+T)$ satisfies $(a, c)$-multiplicative expansion for some constant $a \in (0, 1)$, $c > 1$, if for any $i$ and any subset $A \subset S_i\cup T_i$ with $\Pbb_{\half\rbr{S_i + T_i}}[A] \le a$, we have $\Pbb_{\half(S_i + T_i)}[\Ncal(A)]\ge \min\rbr{c\Pbb_{\half(S_i + T_i)}[A], 1}$.
	
	\item (Constant Expansion) We say $\half(S+ T)$ satisfies $(q, \xi)$-constant expansion for some constant $q, \xi \in (0, 1)$, if for any set $A \subset \Xcal$ with $\Pbb_{\half(S+T)}[A] \ge q$ and $\Pbb_{\half(S_i + T_i)} [A] \le \half, \forall i$, we have $\Pbb_{\half(S+ T)}[\Ncal(A)]\ge \min\rbr{\xi, \Pbb_{\half(S+T)}[A]} + \Pbb_{\half(S+T)}[A]$.
	\end{enumerate}
\end{definition}

The expansion property implicitly states that $S_i$ and $T_i$ are close to each other and regularly shaped. Through the regularizer $R_{\Bcal}(g)$ the label can ``propagate'' from $S_i$ to $T_i$\footnote{Note that our model for subpopulation shift allows any fine-grained form ($m \gg K$), which makes the expansion property more realistic. In image classification, one can take for example $S_i$ as ``Poodles eating dog food'' v.s. $T_i$ as ``Labradors eating meat'' (they're all under the dog class), which is a rather typical form of shift in a real dataset. The representations of such subpopulations can turn out quite close after certain data augmentation and perturbations as in $\Bcal(\cdot)$.}.
One can keep in mind the specific example of Figure~\ref{fig:illustraion} where $B(x) = \{x':\|x - x'\|_2\} \le r\}$ and $S_i \cup T_i$ forms a single connected component. %

Finally, let $G$ be a function class of the learning model. We consider the \emph{realizable} case when the ground truth function $g^\ast \in G$. We assume that the consistency error of the ground truth function is small, and use a constant $\mu > 0$ to represent an upper bound: $R_{\Bcal}(g^\ast) < \mu$. We find the classifier $g$ with the following algorithm:
\begin{align}\nonumber
g &= \argmin_{g: \Xcal \rightarrow \Ycal, g\in G} L_{01}^{S}(g, g_{tc})\\
\label{algorithm_main}
&\mathrm{s.t. }\  R_{\Bcal}(g)  \le \mu ,
\end{align}
where $L_{01}^{S}(g, g_{tc}) := \Pbb_{x\sim S}[g(x)\neq g_{tc}(x)]$ is the 0-1 loss on the \emph{source domain} which encourages $g$ to be aligned with $g_{tc}$ on the source domain. In this paper, we are only concerned with the results of label propagation and not with the optimization process, so we simply take the solution $g$ of (\ref{algorithm_main}) as found and perform analysis on $g$.

Our main theorem will be formulated using $(\half, c)$-multiplicative expansion or $(q, \mu)$-constant expansion\footnote{\citet{wei2021theoretical} contains several examples and illustrations of the expansion property, e.g., the Gaussian mixture example satisfies $(a, c)=(0.5, 1.5)$ multiplicative expansion. The radius $r$ in $\Bcal$ is much smaller than the norm of a typical example, so our model, which requires a separation of $2r$ between components to make $R_{\Bcal}(g^\ast)$ small, is much weaker than a typical notion of ``clustering''.}.

\subsection{Main Theorem}\label{sec:main_theorem}

With the above preparations, we are ready to establish bounds on the target error $\epsilon_T(g) := \Pbb_{x\sim T}(g(x)\neq g^\ast(x))$.

\begin{theorem}[Bound on Target Error with Multiplicative Expansion] \label{thm_multiplicative}
	
	Suppose Assumption~\ref{asm_setting} holds and $\half(S+T)$ satisfies $(\half, c)$-multiplicative expansion. Then the classifier obtained by (\ref{algorithm_main}) satisfies

	$$\epsilon_T(g) \le \max\rbr{\frac{c+1}{c-1}, 3}\frac{8r\mu}{\gamma}.$$
\end{theorem}

\begin{theorem}[Bound on Target Error with Constant Expansion] \label{thm_constant}
	
	Suppose Assumption~\ref{asm_setting} holds and $\half(S+T)$ satisfies $(q, \mu)$-constant expansion.  Then the classifier obtained by (\ref{algorithm_main}) satisfies

	$$\epsilon_T(g) \le (2\max(q, \mu) + \mu)\frac{8r\mu}{\gamma}.$$
\end{theorem}

We make the following remarks on the main results, and also highlight the differences from directly applying \citet{wei2021theoretical} to domain adaptation.

\begin{remark}\label{remark1}
	The theorems state that as long as the ground truth consistency error (equivalently, $\mu$) is small enough, the classifier can achieve near-zero error. This result does not rely on the teacher being close to zero error, as long as the teacher has a positive margin $\gamma$. As a result, the classifier $g$ can improve upon $g_{tc}$ (including on $S$, as the proof of the theorems can show), in a way that the error of $g$ converge to zero as $\mu \rightarrow 0$, regardless of the error of $g_{tc}$. 
	This improvement is due the algorithmic change in Equation \eqref{algorithm_main} which strongly enforces label propagation. Under multiplicative expansion, \citet{wei2021theoretical} attain a bound of the form $O(\frac1c \mathrm{error}(g_{tc}) + \mu)$, which explicitly depends on the accuracy of the teacher $g_{tc}$ on the target domain.\footnote{In the new version of \citet{wei2021theoretical}, the authors proposed a refined result based on iterative training, which alleviates the error bound's dependency on the error of the target domain pseudo-labeler. However, their results still require a mostly correct pseudo-labler on the target domain and require the expansion constant $c$ to be much larger than $1$.} The improvement is due to that we strongly enforce consistency rather than balancing consistency with teacher classifier fit.
\end{remark}

\begin{remark}\label{remark2}
	We do not impose any lower bound on the measure of the components $S_i, T_i$, which is much more general and realistic. From the proofs, one may see that we allow some components to be entirely mislabeled, but in the end, the total measure of such components will be bounded. Directly applying \citet{wei2021theoretical} would require a stringent lower bound on the measure of each $S_i,T_i$.
\end{remark}

\begin{remark}
	We only require expansion with respect to the individual components $S_i \cup T_i$, instead of the entire class~\citep{wei2021theoretical}, which is a weaker requirement. 
\end{remark}

The proofs are essentially because the expansion property turns local consistency into a form of global consistency. The proof sketch is in Section~\ref{sec:proof_sketch}, and the full proof is in Appendix~\ref{sec:full_proof}.

\input{finite_sample}

%% file: finite_sample.tex
\subsection{Finite Sample Guarantee for Deep Neural Networks}\label{sec:finite}
In this section, we leverage existing generalization bounds to prove an end-to-end guarantee on training a deep neural network with finite samples on $S$ and $T$. The results indicate that if the ground-truth class is realizable by a neural network $f^\ast$ by a large robust margin, then the total error can be small.

For simplicity let there be $n$ i.i.d. data each from $S$ and $T$ (a total of $2n$ data), and the empirical distribution is denoted $\hat{S}$ and $\hat{T}$. In order to upper-bound the loss $L_{01}$ and $R_\Bcal(g)$, we apply a notion of all-layer margin \citep{wei2019improved}
\footnote{Though other notions of margin can also work, this helps us to leverage the results from \citet{wei2021theoretical}.},
which measures the stability of the neural net to simultaneous perturbations to each hidden layer. We first cite the useful results from \citet{wei2021theoretical}. Suppose $g(x) = \argmax_{i\in\{1, \cdots, K\}} f(x)_i$ where $f: \Xcal \rightarrow \Rbb^K, x\mapsto W_p\phi(\cdots \phi(W_1x)\cdots)$ is the neural network with weight matrices $\{W_i\}_{i=1}^{p}$,
\footnote{Similarly, $f^\ast$ and $g^\ast$ is the ground truth network and its induced classifier.} 
and $q$ is the maximum width of any layer. Let $m(f, x, y) \ge 0$ denote the all-layer margin at input $x$ for label $y$. 
\footnote{For now, we only use $m(f, x, y) = 0$ if $f(x) \neq y$, so that we can upper bound $\mathbf{1}(g(x) \neq g^\ast(x))$ with $\mathbf{1}\rbr{m(f, x, y) \ge t}$ for any $t > 0$. One can refer the datailed definition to Appendix~\ref{sec:proof_finite} or in \citet{wei2019improved}.
} 
We also define the robust margin $m_\Bcal(f, x) = \min_{x'\in \Bcal(x)}m(f, x', \argmax_i f(x)_i)$. We state the following results.

\begin{proposition}[Theorem C.3 from \citet{wei2021theoretical}]\label{proposition1}
	For any $t > 0$, with probability $1-\delta$, 
	$$\begin{aligned}
	L_{01}^{S}(g, g_{tc}) & \le \Pbb_{x\sim \hat{S}}[m(f, x, g_{tc}(x))\le t] \\
	& + \widetilde{O}\rbr{\frac{\sum_i\sqrt{q}\|W_i\|_F}{t\sqrt{n}} + \sqrt{\frac{\log(1/\delta)+p\log n}{n}}},
	\end{aligned}$$
	where $\widetilde{O}(\cdot)$ hides poly-logarithmic factors in $n$ and $d$.
\end{proposition}

\begin{proposition}[Theorem 3.7 from \citet{wei2021theoretical}]\label{proposition2}
	For any $t > 0$, With probability $1-\delta$, 
	$$\begin{aligned}
	R_{\Bcal}(g) & \le \Pbb_{x\sim \half(\hat{S} + \hat{T})}[m_{\Bcal}(f, x) \le t] \\
	& + \widetilde{O}\rbr{\frac{\sum_i\sqrt{q}\|W_i\|_F}{t\sqrt{n}} + \sqrt{\frac{\log(1/\delta)+p\log n}{n}}}.
	\end{aligned}$$
\end{proposition}

To ensure generalization we replace the loss functions with the margin loss in the algorithm and solve
\begin{align}\nonumber
g & = \argmin_{g: \Xcal \rightarrow \Ycal, g\in G} \Pbb_{x\sim \hat{S}}[m(f, x, g_{tc}(x))\le t]\\
\label{algorithm_margin}
& \textrm{s.t. } 
\Pbb_{x\sim \half(\hat{S} + \hat{T})}[m_{\Bcal}(f, x) \le t]  \le \mu
\end{align}
where $\mu \ge \Pbb_{x\sim \half(\hat{S} + \hat{T})}[m_{\Bcal}(f^\ast, x) \le t]$. Based on these preparations, we are ready to state the final bound.

\begin{theorem}\label{thm:finite_main}
	Suppose Assumption~\ref{asm_setting} holds, and $g$ is returned by (\ref{algorithm_margin}). With probability $1-\delta$, we have:
	
	(a) Under $(1/2, c)$-multiplicative expansion on $\half(S+T)$ we have
	$$\epsilon_T(g) \le \frac{8r}{\gamma}\rbr{\max\rbr{\frac{c+1}{c-1}, 3}\hat{\mu} + \Delta}.$$
	
	(b) Under $(q, \hat{\mu})$-constant expansion on $\half(S+T)$ we have
	$$\epsilon_T(g) \le \frac{8r}{\gamma}\rbr{2\max\rbr{q, \hat{\mu}} + \hat{\mu} + \Delta}.$$
	
	where
	$$\begin{aligned}
	\Delta  = \widetilde{O}& \left(\left(\Pbb_{x\sim\hat{S}}[m(f^\ast, x, g_{tc}(x))\le t] - L_{01}^{\hat{S}}(g^\ast, g_{tc})\right)\right. \\
	& + \left. \frac{\sum_i\sqrt{q}\|W_i\|_F}{t\sqrt{n}} + \sqrt{\frac{\log(1/\delta)+p\log n}{n}} \right),\\
	\hat{\mu} = \mu& + \widetilde{O}\rbr{\frac{\sum_i\sqrt{q}\|W_i\|_F}{t\sqrt{n}} + \sqrt{\frac{\log(1/\delta)+p\log n}{n}}}.\end{aligned}$$
	
\end{theorem} 

\begin{remark}
	Note that the first term in $\Delta$ is small if $t$ is small, and as $n\rightarrow\infty$, the bounds $\Delta$ can be close to $0$ and $\hat{\mu}$ can be close to $\mu$, which gives us the bounds in Section~\ref{sec:main_theorem}.
	
	Similar to the argument in \citet{wei2021theoretical}, it is worth noting that our required sample complexity does not depend exponentially on the dimension. This is in stark contrast to classic non-parametric methods for unknown ``clusters'' of samples, where the sample complexity suffers the curse of dimensionality of the input space.
\end{remark}

The proof of Theorem~\ref{thm:finite_main} is in Appendix~\ref{sec:proof_finite}.

%% file: proof_sketch.tex
\subsection{Proof Sketch for Theorem~\ref{thm_multiplicative} and \ref{thm_constant}}\label{sec:proof_sketch}
To prove the theorems, we first introduce some concepts and notations. 

A point $x\in \Xcal$ is called \emph{robust} w.r.t. $\Bcal$ and $g$ if for any $x'$ in $\Bcal(x)$, $g(x) = g(x')$. Denote 
$$RS(g) := \{x| g(x) = g(x'), \forall x' \in \Bcal(x)\},$$
which is called the \emph{robust set} of $g$. Let 
$$A_{ik} := RS(g)\cap (S_i\cup T_i)\cap\{x|g(x)=k\}$$
for $i\in\{1, \cdots, m\}$, $k\in \{1, \cdots, K\}$, and they form a partition of the set $RS(g)$. Denote 
$$y_i^{\text{Maj}} := \argmax_{k\in\{1, \cdots, K\}}\Pbb_{\half(S+T)}[A_{ik}],$$ 
which is the majority class label of $g$ in the robust set on $(S_i\cup T_i)$. We also call 
$$M_i := \bigcup_{k\in\{1, \cdots, K\}\backslash \{y_i^{\text{Maj}}\}}A_{ik}$$
and $M := \bigcup_{i=1}^{m}M_i$
the \emph{minority robust set} of $g$.
In addition, let
$$\widetilde{M}_i := (S_i\cup T_i) \cap \{x|g(x) \neq y_i^{\text{Maj}}\}$$
and 
$\widetilde{M} := \bigcup_{i=1}^{m}\widetilde{M}_i$
be the \emph{minority set} of $g$, which is superset to the minority robust set.

The expansion property can be used to control the total population of the minority set.

\begin{lemma}[Upper Bound of Minority Set]\label{lem_minority} For the classifier $g$ obtained by (\ref{algorithm_main}), $\Pbb_{\half(S+T)}[\widetilde{M}]$ can be bounded as follows:
	
	(a) Under $(\half, c)$-multiplicative expansion, we have $\Pbb_{\half (S+T)}[\widetilde{M}] \le \max\rbr{\frac{c+1}{c-1}, 3}\mu$. 
	
	(b) Under $(q, \mu)$-constant expansion, we have $\Pbb_{\half (S+T)}[\widetilde{M}] \le 2\max\rbr{q, \mu} + \mu$. 
\end{lemma}

Based on the bound on the minority set, our next lemma says that on most subpopulation components, the inconsistency between $g$ and $g_{tc}$ is no greater than the error of $g_{tc}$ plus a margin $\frac{\gamma}{2}$. Specifically, define 
\begin{align}\nonumber
I = & \{i\in \{1, \cdots, m\}|\\\nonumber
& \Pbb_{x\sim S_i}[g(x)\neq g_{tc}(x)] > \Pbb_{x\sim S_i}[g_{tc}(x)\neq y_i] + \frac{\gamma}{2}\}
\end{align}

and we have the following result

\begin{lemma}[Upper Bound on the Inconsistent Components $I$]\label{lem_I}

	Suppose $\Pbb_{\half(S+T)}[\widetilde{M}] \le C$, then 
	$$\Pbb_S[\cup_{i\in I}S_i] \le \frac{4C}{\gamma}.$$
\end{lemma}

Based on the above results, we are ready to bound the target error $\epsilon_T(g)$.

\begin{lemma}[Bounding the Target Error]\label{lem_target_err}
	Suppose $\Pbb_{\half(S+T)}[\widetilde{M}] \le C$. Let
	$$\epsilon_T^i(g) = \Pbb_T[T_i]\Pbb_{x\sim T}[g(x)\neq y_i]$$
	for $i$ in $\{1, \cdots, m\}$, so that $\epsilon_T(g) = \sum_{i=1}^{m}\epsilon_T^{i}(g)$. Then we can separately bound
	
	(a) $\sum_{i\in I}\epsilon_T^{i}(g)\le \frac{4rC}{\gamma}$

	(b) $\sum_{i\in \{1, \cdots, m\}\backslash I}\epsilon_T^i(g)\le \frac{4rC}{\gamma}$
	
	so that the combination gives
	$$\epsilon_T(g) \le \frac{8rC}{\gamma}.$$
\end{lemma}
Specically, Lemma~\ref{lem_target_err}(a) is obtained by directly using Lemma~\ref{lem_I}, and Lemma~\ref{lem_target_err}(b) is proved by a fine-grained analysis on the minority set.

Finally, we can plug in $C$ from Lemma~\ref{lem_minority} and the desired main results are obtained.

%% file: label_prop_generalized.tex
\section{Label Propogation in Generalized Subpopulation Shift}\label{sec:label_prop_generalized}
In this section, we show that the previous label propagation algorithm can be applied to a much more general setting than standard unsupervised domain adaptation. In a word, as long as we perform consistency regularization on an unlabeled dataset that covers both the teacher classifier's domain and the target domain, we can perform label propagation through the subpopulation of the unlabeled data.

Specifically, we still let $S$ be the source distribution where we have a teacher $g_{tc}$ on, and $T$ is the target distribution. The difference is that we have a ``covering'' distribution $U$ (Assumption~\ref{asm_setting_2}(c)) where we only make use of unlabeled data, and the expansion property is assumed to hold on $U$.

\begin{assumption}\label{asm_setting_2}
Assume the distributions are of the following structure: $\supp(S) = \cup_{i=1}^{m}S_i$, $\supp(T) = \cup_{i=1}^{m}T_i$, $\supp(U) = \cup_{i=1}^{m}U_i$, where $U_i \cap U_j = \emptyset$ for $i \neq j$, and $S_i \cup T_i\subset U_i$. Again, assume the ground truth class $g^\ast(x)$ for $x\in U_i$ is consistent (constant), denoted $y_i$.\ We abuse the notation to let $S_i$, $T_i$, $U_i$ also denote the conditional distribution of $S, T, U$ on the set $S_i, T_i, U_i$ respectively. We also make the following assumptions, with an additional (c) that says $U$ ``covers'' $S, T$.

(a)(b): Same as Assumption~\ref{asm_setting}(a)(b).

\textbf{(c) There exists a constant $\kappa \ge 1$ such that the measure $S_i$, $T_i$ are bounded by $\kappa U_i$. That is, for any $A \subset\Xcal$, 
$$\Pbb_{S_i}(A) \le \kappa \Pbb_{U_i}(A)\text{ and } \Pbb_{T_i}(A) \le \kappa \Pbb_{U_i}(A).$$}
\end{assumption}
The regularizer now becomes
$$R_{\Bcal}(g) := \Pbb_{x\sim U}[\exists x' \in \Bcal(x), \textrm{s.t. }g(x)\neq g(x')].$$
On can see that the main difference is that we replaced $\frac{1}{2}(S+T)$ from the previous domain adaptation with a general distribution $U$. Indeed, we assume expansion on $U$ and can establish bounds on $\epsilon_T(g)$.

\begin{definition}[Expansion on $U$] \label{def_expansion_2}

(1) We say $U$ satisfies $(a, c)$-multiplicative expansion for some constant $a \in (0, 1)$, $c > 1$, if for any $i$ and any subset $A \subset U$ with $\Pbb_{U_i}[A] \le a$, we have $\Pbb_{U_i}[\Ncal(A)]\ge \min\rbr{c\Pbb_{U_i}[A], 1}$.

(2) We say $U$ satisfies $(q, \xi)$-constant expansion for some constant $q, \xi \in (0, 1)$, if for any set $A \subset \Xcal$ with $\Pbb_{U_i}[A] \ge q$ and $\Pbb_{U_i} [A] \le \half, \forall i$, we have $\Pbb_{U}[\Ncal(A)]\ge \min\rbr{\xi, \Pbb_{U}[A]} + \Pbb_{U}[A]$.
\end{definition}

\begin{theorem}[Bound on Target Error with Multiplicative Expansion, Generalized] \label{thm_multiplicative_2}

Suppose Assumption~\ref{asm_setting_2} holds and $U$ satisfies $(\half, c)$-multiplicative expansion. Then the classifier obtained by (\ref{algorithm_main}) satisfies
$$\epsilon_T(g) \le \max\rbr{\frac{c+1}{c-1}, 3}\frac{4\kappa r\mu}{\gamma}.$$
\end{theorem}

\begin{theorem}[Bound on Target Error with Constant Expansion, Generalized] \label{thm_constant_2}

Suppose Assumption~\ref{asm_setting} holds and $U$ satisfies $(q, \mu)$-constant expansion.  Then the classifier obtained by (\ref{algorithm_main}) satisfies
$$\epsilon_T(g) \le (2\max(q, \mu) + \mu)\frac{4\kappa r\mu}{\gamma}.$$
\end{theorem}

Choosing special cases of the structure $U$, we can naturally obtain the following special cases that correspond to the models shown in Figure~\ref{figure2}.

\begin{figure*}[ht]
    \centering
        \includegraphics[width=1\textwidth]{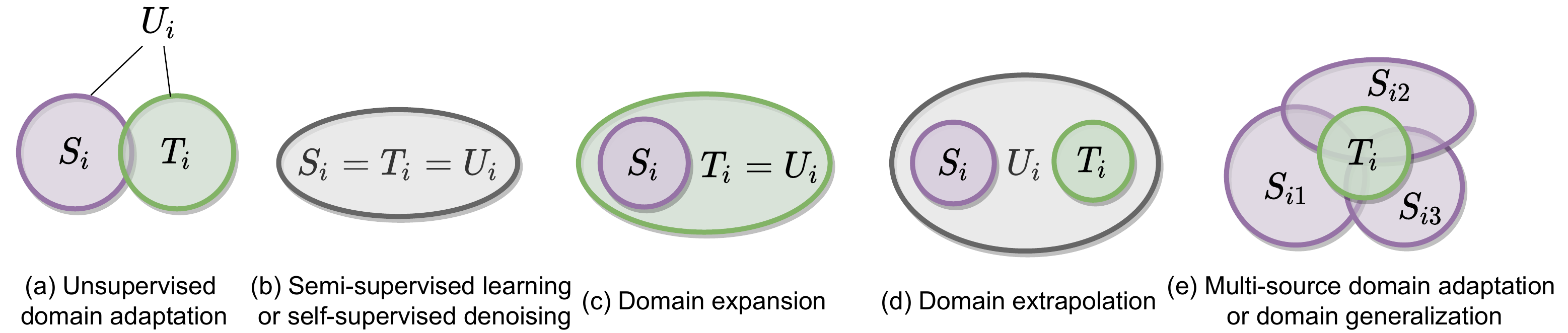}%
    \caption{Settings of generalized subpopulation shift in Section~\ref{sec:label_prop_generalized}. The figures only draw one subpopulation $i$ for each model.} %
    \label{figure2}
\end{figure*}

\begin{enumerate}
\item \textbf{Unsupervised domain adaptation} (Figure~\ref{figure2}(a)).
When $U_i = \frac{1}{2}(S_i+T_i)$, we immediately obtain the results in Section~\ref{sec:main_theorem} by plugging in $\kappa = 2$. Therefore, Theorem~\ref{thm_multiplicative} and ~\ref{thm_constant} is just a special case of Theorem~\ref{thm_multiplicative_2} and \ref{thm_constant_2}. 

\item \textbf{Semi-supervised learning or self-supervised denoising} (Figure~\ref{figure2}(b)).
When $S_i = T_i = U_i$, the framework becomes the degenerate version of learning a $g$ from a $g_{tc}$ in a single domain. $g_{tc}$ can be a pseudo-labeler in the semi-supervised learning or some other pre-trained classifier self-supervised denoising. Our results improve upon \citet{wei2021theoretical} under this case as discussed in Remark~\ref{remark1}, \ref{remark2}.

\item \textbf{Domain expansion} (Figure~\ref{figure2}(c)).
When $T_i = U_i$, this becomes a problem between semi-supervised learning and domain adaptation, and we call it domain expansion. That is, the source $S$ is a sub-distribution of $T$ where we need to perform well. Frequently, we have a big unlabeled dataset and the labeled data is only a specifc part. 

\item \textbf{Domain extrapolation} (Figure~\ref{figure2}(d)).
When $S_i \cup T_i$ does not satisfy expansion by itself, e.g. they are not connected by $\Bcal(\cdot)$, but they are connected through $U_i$, we can still obtain small error on $T$. We term this kind of task domain extrapolation, where we have a small source and small target distribution that is not easy to directly correlate, but is possible through a third and bigger unlabeled dataset $U$ where label information can propagate.
\item \textbf{Multi-Source domain adaptation or domain generalization} (Figure~\ref{figure2}(e)).
We have multiple source domains and take $U$ as the union (average measure) of all source domains. Learning is guaranteed if in the input space or some representation space, $U$ can successfully ``cover'' $T$, the target distribution in multi-source domain adaptation or the test distribution in domain generalization. Also, as the framework suggests, we do not require all the source domains to be labeled, depending on the specific structure. 
\end{enumerate}

The general label propogation framework proposed in this section is widely applicable in many practical scenarios, and would also be an interesting future work. The full proof of the theorems in this section is in Appendix~\ref{sec:full_proof}.

%% file: experiments.tex
\section{Experiments}
\setcounter{table}{1}
\begin{table*}[h]
\begin{center}
	\scriptsize
\begin{tabular}{@{}cccccccc@{}}

\toprule
Method       & { A $\to$ W} & { D $\to$ W} & { W $\to$ D} & { A $\to$ D} & { D $\to$ A} & { W $\to$ A} & Average \\ \midrule
MDD          & 94.97$\pm$0.70                                       & 98.78$\pm$0.07                                       & 100$\pm$0                                            & 92.77$\pm$0.72                                       & 75.64$\pm$1.53                                       & 72.82$\pm$0.52                                       & 89.16   \\
MDD+FixMatch & 95.47$\pm$0.95                                       & 98.32$\pm$0.19                                       &                   100$\pm$0                                   & 93.71$\pm$0.23                                       & 76.64$\pm$1.91                                       & 74.93$\pm$1.15                                       & \textbf{89.84}   \\ \bottomrule
\end{tabular}
\caption{\label{tab:office31}Performance of MDD and MDD+FixMatch on Office-31 dataset.}
\end{center}

\resizebox{\textwidth}{!}{
\setlength{\tabcolsep}{0.8mm}{
\begin{tabular}{@{}cccccccccccccc@{}}
\toprule
Method       & Ar $\to$ Cl        & Ar $\to$ Pr        & Ar $\to$ Rw        & Cl $\to$ Ar        & Cl $\to$ Pr        & Cl $\to$ Rw        & Pr $\to$ Ar        & Pr $\to$ Cl        & Pr $\to$ Rw        & Rw $\to$ Ar         & Rw $\to$ Cl        & Rw $\to$ Pr        & Average \\ \midrule
MDD          & 54.9$\pm$0.7 & 74.0$\pm$0.3 & 77.7$\pm$0.3 & 60.6$\pm$0.4 & 70.9$\pm$0.7 & 72.1$\pm$0.6 & 60.7$\pm$0.8 & 53.0$\pm$1.0 & 78.0$\pm$0.2 & 71.8$\pm$0.4 & 59.6$\pm$0.4 & 82.9$\pm$0.3 &   68.0      \\
MDD+FixMatch & 55.1$\pm$0.9 & 74.7$\pm$0.8 & 78.7$\pm$0.5 & 63.2$\pm$1.3 & 74.1$\pm$1.8 & 75.3$\pm$0.1 & 63.0$\pm$0.6 & 53.0$\pm$0.6 & 80.8$\pm$0.4 &  73.4$\pm$0.1  &  59.4$\pm$0.7  &   84.0$\pm$0.5   &   \textbf{69.6} \\ \bottomrule
\end{tabular}
}}
\caption{\label{tab:office_home}Performance of MDD and MDD+FixMatch on Office-Home dataset.}
\end{table*}
\setcounter{table}{0}
\label{sec:exps}
In this section, we first conduct experiments on a dataset that is constructed to simulate natural subpopulation shift. Then we generalize the aspects of subpopulation shift to classic unsupervised domain adaptation datasets by combining distributional matching methods and consistency-based label propagation method.
\subsection{Subpopulation Shift Dataset}
We empirically verify that label propagation via consistency regularization works well for subpopulation shift tasks. Towards this goal, we constructed an Unsupervised Domain Adaptation (UDA) task using the challenging ENTITY-30 task from BREEDS tasks~\citep{santurkar2021breeds}, and directly adapt FixMatch~\citep{sohn2020fixmatch}, an existing consistency regularization method for \emph{semi-supervised learning} to the subpopulation shift tasks. The main idea of FixMatch is to optimize the supervised loss on weak augmentations of source samples, plus consistency regularization, which encourages the prediction of the classifier on strong augmentations of a sample to be the same to the prediction on weak augmentations of the sample\footnote{Empirically, FixMatch also combines self-training techniques that take the hard label of the prediction on weak augmentations. We also use Distribution Alignment~\citep{berthelot2019remixmatch} mentioned in Section 2.5 of the FixMatch paper.}. In contrast to semi-supervised learning where the supports of unlabeled data and labeled data are inherently the same, in subpopulation shift problems, the support sets of different domains are disjoint. To enable label propagation, we need a good feature map to enable label propagation on the \emph{feature space}. We thus make use of the feature map learned by a self-supervised learning algorithm SwAV~\citep{caron2020unsupervised}, which simultaneously \emph{clusters} the data while \emph{enforcing consistency} between cluster assignments produced for different augmentations of the same image. This representation has two merits; first, it encourages subpopulations with similar representations to cluster in the feature space. Second, it enforces the augmented samples to be close in the feature space. We expect that subclasses from the same superclass will be assigned to the same cluster and thus enjoy the expansion property to a certain extent in the feature space. We defer the detailed experimental settings to Appendix~\ref{appsec:detail_exp} and report the results here.

\begin{table}[H]
	\label{tab:exps}
	\centering
	\begin{tabular}{cll}
		\toprule
		Method          & Source Acc & Target Acc \\ 
		\midrule
		Train on Source &  91.91$\pm$0.23          &     56.73$\pm$0.32       \\
		DANN~\citep{ganin2016domain}            &     92.81$\pm$0.50       &      61.03$\pm$4.63      \\
		MDD~\citep{zhang2019bridging}             &      92.67$\pm$0.54      &      63.95$\pm$0.28      \\
		FixMatch~\citep{sohn2020fixmatch}        &     90.87$\pm$0.15       &     72.60$\pm$0.51    \\
		\bottomrule  
	\end{tabular}
	\caption{Comparison of performance on ENTITY-30 (Acc refers to accuracy which is measured by percentage).}
\end{table}

We compare the performance of the adaptation of FixMatch with popular distributional matching methods, i.e., DANN~\citep{ganin2016domain} and MDD~\citep{zhang2019bridging}\footnote{We use the implementation from \citet{dalib}, which shows that MDD has the best performance among the evaluated methods.}. For a fair comparison, all models are finetuned from SwAV representation. As shown in Table~\ref{tab:exps}, the adaptation with FixMatch obtains significant improvement upon the baseline method that only trains on the source domain by more than $15\%$ points on the target domain. FixMatch also outperforms distributional matching methods by more than $8\%$. The results suggest that unlike previous distributional matching-based methods, consistency regularization-based methods are preferable on domain adaptation tasks when encountering subpopulation shift. This is also aligned with our theoretical findings.

\subsection{Classic Unsupervised Domain Adaptation Datasets}
In this section we conduct experiments on classic unsupervised domain adaptation datasets, i.e., Office-31~\citep{saenko2010adapting}, Office-Home~\citep{venkateswara2017deep}, where source and target domains mainly differ in style, e.g., artistic images to real-world images. Distributional matching methods seek to learn an invariant representation which removes confounding information such as the style. Since the feature distributions of different domains are encouraged to be matched, the supports of different domains in the feature space are overlapped which enables label propagation. In addition, subpopulation shift from source to target domain may remain even if the styles are unified in the feature space. This inspires us to combine distributional matching methods and label propagation.

As a preliminary attempt, we directly combine MDD~\citep{zhang2019bridging} and FixMatch~\citep{sohn2020fixmatch} to see if there is gain upon MDD. Specifically, we first learn models using MDD on two classic unsupervised domain adaptation datasets, Office-31 and Office-Home. Then we finetune the learned model using FixMatch (with Distribution Alignment extension as described in previous subsection). The results in Table~\ref{tab:office31}, \ref{tab:office_home} confirm that finetuning with FixMatch can improve the performance of MDD models. The detailed experimental settings can be found in Appendix~\ref{appsec:detail_exp}.

%% file: conclusion.tex
\section{Conclusion}\label{sec:conclusion}

In this work, we introduced a new theoretical framework of learning under subpopulation shift through label propagation, providing new insights on solving domain adaptation tasks. We provided accuracy guarantees on the target domain for a consistency regularization-based algorithm using a fine-grained analysis under the expansion assumption. Our generalized label propagation framework in Section~\ref{sec:label_prop_generalized} subsumes the previous domain adaptation setting and also provides an interesting direction for future work.

%% file: full_proof.tex
\section{Proof of Theorem~\ref{thm_multiplicative}, \ref{thm_constant}, \ref{thm_multiplicative_2}, and \ref{thm_constant_2}} \label{sec:full_proof}

Note that in Section~\ref{sec:label_prop_generalized}, by taking $U = \half(S+T)$, in Assumption~\ref{asm_setting_2}(c) we have $\kappa = 2$. By plugging in $\kappa$, Theorem~\ref{thm_multiplicative} and \ref{thm_constant} immediately becomes the corollary of Theorem~\ref{thm_multiplicative_2} and \ref{thm_constant_2}. Therefore, we only provide a full proof for Theorem~\ref{thm_multiplicative_2} and \ref{thm_constant_2} here.

First, similar to Section~\ref{sec:proof_sketch}, we give a proof sketch for Theorem~\ref{thm_multiplicative_2} and \ref{thm_constant_2}, which includes the corresponding definitions and lemmas for this generalized setting.

\subsection{Proof Sketch for Theorem~\ref{thm_multiplicative_2} and \ref{thm_constant_2}}\label{sec:proof_sketch_2}
To prove the theorems, we first introduce some concepts and notations. 

A point $x\in \Xcal$ is called \emph{robust} w.r.t. $\Bcal$ and $g$ if for any $x'$ in $\Bcal(x)$, $g(x) = g(x')$. Denote 
$$RS(g) := \{x| g(x) = g(x'), \forall x' \in \Bcal(x)\},$$
which is called the \emph{robust set} of $g$. Let 
$$A_{ik} := RS(g)\cap U_i\cap\{x|g(x)=k\}$$
for $i\in\{1, \cdots, m\}$, $k\in \{1, \cdots, K\}$, and they form a partition of the set $RS(g)$. Denote 
$$y_i^{\text{Maj}} := \argmax_{k\in\{1, \cdots, K\}}\Pbb_{U}[A_{ik}],$$ 
which is the majority class label of $g$ in the robust set on $U_i$. We also call 
$$M_i := \bigcup_{k\in\{1, \cdots, K\}\backslash \{y_i^{\text{Maj}}\}}A_{ik}$$
and $M := \bigcup_{i=1}^{m}M_i$
the \emph{minority robust set} of $g$.
In addition, let
$$\widetilde{M}_i := U_i \cap \{x|g(x) \neq y_i^{\text{Maj}}\}$$
and 
$\widetilde{M} := \bigcup_{i=1}^{m}\widetilde{M}_i$
be the \emph{minority set} of $g$, which is superset to the minority robust set.

The expansion property can be used to control the total population of the minority set.

\begin{lemma}[Upper Bound of Minority Set]\label{lem_minority_2} For the classifier $g$ obtained by (\ref{algorithm_main}), $\Pbb_{U}[\widetilde{M}]$ can be bounded as follows:
	
	(a) Under $(\half, c)$-multiplicative expansion, we have $\Pbb_{U}[\widetilde{M}] \le \max\rbr{\frac{c+1}{c-1}, 3}\mu$. 
	
	(b) Under $(q, \mu)$-constant expansion, we have $\Pbb_{U}[\widetilde{M}] \le 2\max\rbr{q, \mu} + \mu$. 
\end{lemma}

Based on the bound on the minority set, our next lemma says that on most subpopulation components, the inconsistency between $g$ and $g_{tc}$ is no greater than the error of $g_{tc}$ plus a margin $\frac{\gamma}{2}$. Specifically, define 
\begin{align}\nonumber
I = & \{i\in \{1, \cdots, m\}|\\\nonumber
& \Pbb_{x\sim S_i}[g(x)\neq g_{tc}(x)] > \Pbb_{x\sim S_i}[g_{tc}(x)\neq y_i] + \frac{\gamma}{2}\}
\end{align}
and we have the following result

\begin{lemma}[Upper Bound on the Inconsistent Components $I$]\label{lem_I_2}
	Suppose $\Pbb_{U}[\widetilde{M}] \le C$, then 
	$$\Pbb_S[\cup_{i\in I}S_i] \le \frac{2\kappa C}{\gamma}.$$
\end{lemma}

Based on the above results, we are ready to bound the target error $\epsilon_T(g)$.

\begin{lemma}[Bounding the Target Error]\label{lem_target_err_2}
	Suppose $\Pbb_{U}[\widetilde{M}] \le C$. Let
	$$\epsilon_T^i(g) = \Pbb_T[T_i]\Pbb_{x\sim T}[g(x)\neq y_i]$$
	for $i$ in $\{1, \cdots, m\}$, so that $\epsilon_T(g) = \sum_{i=1}^{m}\epsilon_T^{i}(g)$. Then we can separately bound
	
	(a) $\sum_{i\in I}\epsilon_T^{i}(g)\le \frac{2\kappa rC}{\gamma}$
	
	(b) $\sum_{i\in \{1, \cdots, m\}\backslash I}\epsilon_T^i(g)\le \frac{2\kappa rC}{\gamma}$
	
	so that the combination gives
	$$\epsilon_T(g) \le \frac{4\kappa rC}{\gamma}.$$
\end{lemma}
Specically, Lemma~\ref{lem_target_err_2}(a) is obtained by directly using Lemma~\ref{lem_I_2}, and Lemma~\ref{lem_target_err_2}(b) is proved by a fine-grained analysis on the minority set.

Finally, we can plug in $C$ from Lemma~\ref{lem_minority_2} and the desired results in Theorem~\ref{thm_multiplicative_2} and \ref{thm_constant_2} are obtained.

To make the proof complete, we provide a detailed proof of Lemma~\ref{lem_minority_2}, \ref{lem_I_2}, \ref{lem_target_err_2} in the following subsections.

\subsection{Proof of Lemma~\ref{lem_minority_2}.}\label{sec:proof_lem_1}
\begin{proof}
	We first prove the $(q, \mu)$-expansion case (a). The probability function $\Pbb$ are all w.r.t. the distribution $U$ in this lemma, so we omit this subscript. We also use $\Pbb_{i}$ for $\Pbb_{U_i}$ in this lemma.
	
	The robust minority set is $M_i = \cup_{k\in \{1, \cdots, K\}\backslash \{y_i^{\text{Maj}}\}}A_{ik}$. In order to do expansion, we partition $M_i$ into two halves:
	
	\begin{lemma}[Partition of $M_i$] \label{lem_partition_of_Mi}
		For each $i\in\{1, \cdots, m\}$, there exists a partition of the set $\{1, \cdots, K\}\backslash \{y_i^{\text{Maj}}\}$ into $J_{i1}$ and $J_{i2}$ such that the corresponding partition $M_i = M_{i}^1\cup M_{i}^2$ ($M_{i}^1 = \cup_{k\in J_{i1}}A_{ik}$, $M_{i}^2 = \cup_{k\in J_{i2}}A_{ik}$) satisfies $\Pbb_i[M_i^1] \le \half$ and $\Pbb_i[M_i^2] \le \half$.
	\end{lemma}
	
	\begin{proof}
		Starting from $J_{i1} = J_{i2} = \emptyset$, and each time we add an element $k_0\in\{1, \cdots, K\}\backslash \{y_i^{\text{Maj}}\}$ into $J_{i1}$ or $J_{i2}$ while keeping the properties $\Pbb_i[M_{i}^{1}] \le \half$ and $\Pbb_i[M_{i}^2] \le \half$ hold. We prove that for any $k_0\in\{1, \cdots, K\}\backslash (\{y_i^{\text{Maj}}\} \cup J_{i1}\cup J_{i2})$, either $\Pbb_i[\cup_{k\in J_{i1} \cup \{k_0\}}A_{ik}] \le \frac{1}{2}$ or $\Pbb_i[\cup_{k\in J_{i2} \cup \{k_0\}}A_{ik}] \le \frac{1}{2}$ holds, so we can repeat the process until $J_{i1}$ and $J_{i2}$ is a partition of $\{1, \cdots, K\}\backslash \{y_i^{\text{Maj}}\}$. In fact, since 
		\begin{align*}
		& \;\;\;\;\;  \Pbb_i[\cup_{k\in J_{i1} \cup \{k_0\}}A_{ik}] + \Pbb_i[\cup_{k\in J_{i2} \cup \{k_0\}}A_{ik}] \\
		& \le \Pbb_i[\cup_{k\in J_{i1} \cup \{k_0\}}A_{ik}] + \Pbb_i[\cup_{k\in J_{i2} \cup \{y_i^{\text{Maj}}\}}A_{ik}]\\
		& \;\;\;\;\; \text{(By the definition of $y_i^{\text{Maj}}$)}\\
		& \le \Pbb_i[\cup_{k\in\{1, \cdots, K\}}A_{ik}]\\
		& \le 1,
		\end{align*}
		we know that either $\Pbb_i[\cup_{k\in J_{i1} \cup \{k_0\}}A_{ik}]$ or $\Pbb_i[\cup_{k\in J_{i2} \cup \{k_0\}}A_{ik}]$ is no more than $\half$, and Lemma~\ref{lem_partition_of_Mi} is proved.
		
	\end{proof}
	
	Let $M^1 = \cup_{i=1}^{m}M_{i}^1$, $M^2 = \cup_{i=1}^{m}M_{i}^2$, so that $M^1$ and $M^2$ form a partition of $M$. Based on Lemma~\ref{lem_partition_of_Mi}, we know that either $\Pbb[M^1] < q$, or $M^1$ satisfies the requirement for $(q, \mu)$-constant expansion. Hence,
	\begin{align}\label{equ_const_expansion_minority}
	& \Pbb[\Ncal(M^1)] \ge \Pbb[M^1] + \min\rbr{\mu, \Pbb[M^1]} \text{ or } \Pbb[M^1] < q 
	\end{align}
	
	On the other hand, we claim that $\Ncal(M^1)\backslash M^1$ contains only non-robust points. Otherwise, suppose there exists a robust point $x \in \Ncal(M^1)\backslash M^1$, say $x \in \Ncal(A_{ik})$ for some $i\in \{1, \cdots, m\}$ and $k\in J_{i1}$. By the definition of neighborhood, there exists $x' \in A_{ik}$ such that there exists $x'' \in \Bcal(x) \cap \Bcal(x')$. Therefore, by the definition of robustness, $g(x) = g(x'') = g(x') = k$. Also by the definition of neighborhood, we know that $x \in U_i$, so it must be that $x\in A_{ik}$ since $x$ is robust. This contradicts with $x\notin M^1$! Therefore, $\Ncal(M^1)\backslash M^1$ is a subset of the set of all non-robust points.
	Since the total measure of non-robust points is $R_\Bcal(g)$ by definition, we know that 
	\begin{align}\label{equ_2}
	\Pbb[\Ncal(M^{1})] - \Pbb[M^1] \le \Pbb[\Ncal(M^{1})\backslash M^{1}] \le R_\Bcal(G) < \mu.
	\end{align}
	
	Combining (\ref{equ_const_expansion_minority}) and (\ref{equ_2}), we know that under $\Pbb[M^1] \ge q$, it must hold that $\Pbb[M^1] < \mu$, or else (\ref{equ_const_expansion_minority}) and (\ref{equ_2}) would be a contradiction. In all, this means that $\Pbb[M^1] \le \max(q, \mu)$ in any case. 
	
	Similarly, we know $\Pbb[M^2] \le \max(q, \mu)$ also hold. Therefore, $\Pbb[M] \le 2\max(q, \mu)$.
	
	Since $\widetilde{M} \backslash M$ only consists of non-robust points, we know that
	$$\Pbb[\widetilde{M}] \le \Pbb[M] + R_B(g) \le 2\max\rbr{q, \mu} + \mu,$$
	which is the desired result (a).
	
	For the $(\half, c)$-multiplicative expansion case (b), it is easy to verify that $(\half, c)$-multiplicative expansion must imply $(\frac{\mu}{c-1}, \mu)$-constant expansion (See Lemma B.6 in \citet{wei2021theoretical}). Therefore, the result is obtained by plugging in $q = \frac{\mu}{c-1}$. 
\end{proof}

\subsection{Proof of Lemma~\ref{lem_I_2}}\label{sec:proof_lem_2}
\begin{proof}
	We first prove the following lemma.
	
	\begin{lemma}\label{lem_smallresult}
		For any $i\in\{1, \cdots, m\}$, we have
		\begin{align}\label{equ_smallresult}
		\Pbb_{x\sim S_i}[g(x) \neq g_{tc}(x)] + \Pbb_{S_i}[\widetilde{M_i}] \ge \Pbb_{x\sim S_i}[g_{tc}(x)\neq y_i].
		\end{align}
	\end{lemma}
	
	\begin{proof}
		Based on the margin assumption (Assumption~\ref{asm_setting_2}(a)), we know that $\Pbb_{x\sim S_i}[g_{tc}(x) = y_i] \ge \Pbb_{x\sim S_i}[g_{tc}(x) = y_i^{\text{Maj}}]$. Therefore, $\Pbb_{x\sim S_i}[g_{tc}(x) \neq y_i] \le \Pbb_{x\sim S_i}[g_{tc}(x) \neq y_i^{\text{Maj}}]$. Along with triangle inequality we know that
		$$\begin{aligned}
		& \;\;\;\;\, \Pbb_{x\sim S_i}[g(x) \neq g_{tc}(x)] + \Pbb_{S_i}[\widetilde{M_i}]\\
		& = \Pbb_{x\sim S_i}[g(x) \neq g_{tc}(x)] + \Pbb_{x\sim S_i}[g(x) \neq y_i^{\text{Maj}}] \\
		& \ge \Pbb_{x\sim S_i}[g_{tc}(x) \neq y_i^{\text{Maj}}] \\
		& \ge \Pbb_{x\sim S_i}[g_{tc}(x) \neq y_i],
		\end{aligned}$$
		which proves the result.
	\end{proof}
	
	Based on Lemma~\ref{lem_smallresult}, we can write:
	\begin{align}\nonumber
	& \;\;\;\;\;L_{01}^S(g, g_{tc}) \\\nonumber
	& = 
	\sum_{i\in I}\Pbb_S[S_i]\Pbb_{x\sim S_i}[g(x) \neq g_{tc}(x)] \\\nonumber
	& \;\;\;\; + \sum_{i\in \{1, \cdots, m\}\backslash I}\Pbb_S[S_i]\Pbb_{x\sim S_i}[g(x) \neq g_{tc}(x)] \\\nonumber
	& \ge \sum_{i\in I}\Pbb_S[S_i]\rbr{\Pbb_{x\sim S_i}[g(x) \neq y_i] + \frac{\gamma}{2}} \\\nonumber
	& \;\;\;\; + \sum_{i\in \{1, \cdots, m\}\backslash I}\Pbb_S[S_i]\sbr{\Pbb_{x\sim S_i}[g_{tc}(x)\neq y_i] - \Pbb_{S_i}[\widetilde{M_i}]}\\\nonumber
	& \;\;\;\;\text{(by definition of $I$ for the first term, }\\\nonumber
	& \;\;\;\;\;\; \text{by Lemma~\ref{lem_smallresult} for the second term)}\\\nonumber
	& = L_{01}^S(g^\ast, g_{tc}) + \frac{\gamma}{2}\sum_{i\in I}\Pbb_S[S_i] \\\label{equ:proof_lem_2}
	& \;\;\;\; - \sum_{i\in \{1, \cdots, m\}\backslash I}\Pbb_S[\widetilde{M}_i].
	\end{align}
	
	Since by definition of our algorithm, $L_{01}^S(g, g_{tc}) \le L_{01}^S(g^\ast, g_{tc})$, we finally know that 
	$$\begin{aligned}
	\sum_{i\in I}\Pbb_S[S_i] & \le \frac{2}{\gamma}\sum_{i\in \{1, \cdots, m\}\backslash I}\Pbb_S[\widetilde{M}_i]\\
	& \le \frac{2}{\gamma}\Pbb_S[\widetilde{M}]\\
	& \le \frac{2\kappa}{\gamma}\Pbb_{U}[\widetilde{M}] \;\;\;\;\text{(by Assumption~\ref{asm_setting_2}(c))}\\
	& \le \frac{2\kappa C}{\gamma},
	\end{aligned}$$
	which is the desired result.
\end{proof}

\subsection{Proof of Lemma~\ref{lem_target_err_2}.} \label{sec:proof_lem_3}
\begin{proof}
	(a) This is a direct result from Lemma~\ref{lem_I_2} since
	\begin{align}\nonumber
	\sum_{i\in I}\epsilon_T^i(g) 
	& \le \sum_{i\in I} \Pbb_T[T_i] \\\nonumber
	& \le \sum_{i\in I} r\Pbb_S[S_i] \\\label{equ:proof_lem_3}
	& \le \frac{2\kappa rC}{\gamma}.
	\end{align}
	
	(b) For $i\in\{1, \cdots, m\}\backslash I$, we proceed by considering the following two cases: $y_i = y_i^{\text{Maj}}$ or $y_i \neq y_i^{\text{Maj}}$.
	
	If $y_i = y_i^{\text{Maj}}$, we have
	$$\epsilon_T^i(g) = \Pbb_T[\widetilde{M}_i] \le \kappa\Pbb_{U}[\widetilde{M}_i].$$
	
	If $y_i \neq y_i^{\text{Maj}}$, we have
	$$\begin{aligned}
	\frac{\Pbb_S[\widetilde{M}_i]}{\Pbb_S[S_i]} 
	& = \Pbb_{S_i}[\widetilde{M}_i]\\
	& = \Pbb_{x\sim S_i}[g(x)\neq y_i^{\text{Maj}}]\\
	& \ge \Pbb_{x\sim S_i}[g_{tc}(x)\neq y_i^{\text{Maj}}] - \Pbb_{x\sim S_i}[g_{tc}(x)\neq g(x)]\\
	& \text{(triangle inequality)}\\
	& = 1 - \Pbb_{x\sim S_i}[g_{tc}(x)= y_i^{\text{Maj}}]\\
	&\;\;\;\; - \Pbb_{x\sim S_i}[g_{tc}(x)\neq g(x)]\\
	& \ge 1 - (\Pbb_{x\sim S_i}[g_{tc}(x)= y_i] - \gamma) \\
	&\;\;\;\; - (\Pbb_{x\sim S_i}[g_{tc}(x) \neq y_i] + \frac{\gamma}{2})\\
	& \text{(Assumption~\ref{asm_setting_2}(a) and Definition of $I$)}\\
	& = \frac{\gamma}{2}.
	\end{aligned}$$
	
	Then we have
	$$\begin{aligned}
	\epsilon_T^i(g) & \le \Pbb_T[T_i] \\
	& \le r\Pbb_S[S_i] \\
	& \le \frac{2r}{\gamma}\Pbb_S[\widetilde{M}_i] \\
	& \le \frac{2\kappa r}{\gamma}\Pbb_{U}[\widetilde{M}_i]
	\end{aligned}$$
	
	Summarizing the two cases above, since $\frac{2 \kappa r}{\gamma} \ge 2$ must hold, we always have
	
	$$\epsilon_T^i(g) \le \frac{2\kappa r}{\gamma}\Pbb_{U}[\widetilde{M}_i],$$
	and as a result,
	
	$$\begin{aligned}
	\sum_{i\in\{1, \cdots, m\}\backslash I}\epsilon_T^i(g) & \le \sum_{i\in\{1, \cdots, m\}\backslash I}\frac{2\kappa r}{\gamma}\Pbb_{U}[\widetilde{M}_i] \\
	& \le \frac{2\kappa r}{\gamma}\Pbb_{U}[\widetilde{M}] \\
	& = \frac{2\kappa rC}{\gamma}.
	\end{aligned}$$
	
\end{proof}

%% file: proof_finite.tex
\section{Proof of Results in Section~\ref{sec:finite}}
\label{sec:proof_finite}

As a side note, we first state the definition of the all-layer margin from \citet{wei2019improved}. For the neural network $f(x) = W_p\phi(\cdots \phi(W_1x)\cdots)$, we write $f$ as $f(x) = f_{2p-1}\circ \cdots \circ f_1(x)$, where the $f_i$'s alternate between matrix multiplications and applications of the activation function $\phi$. Let $\delta_1, \cdots, \delta_{2p-1}$ denote perturbations intended to be applied at each layer $i = 1, \cdots, 2p-1$, and the perturbed network output $f(x, \delta_1, \cdots, \delta_{2p-1})$ is recursively defined as
\begin{align*}
h_1(x, \delta) & = f_1(x) + \delta_1\|x\|_2,\\
h_i(x, \delta) & = f_i(h_{i-1}(x, \delta)) + \delta_i\|h_{i-1}(x, \delta)\|_2,\\
f(x, \delta) & = h_{2p-1}(x, \delta).
\end{align*}
And the all-layer margin is defined as the minimum norm of $\delta$ required to make the classifier misclassify the input, i.e.
\begin{align*}
m(f, x, y) &:= \min_{\delta_1, \cdots, \delta_{2p-1}}\sqrt{\sum_{i=1}^{2p-1}\|\delta_i\|_2^2}\\
& \text{subject to } \argmax_{y'} f(x, \delta_1, \cdots, \delta_{2p-1})_{y'} \neq y.
\end{align*}
The related results about all-layer margin (Proposition~\ref{proposition1} and \ref{proposition2}), though, come directly from \citet{wei2021theoretical}.

\subsection{Proof of Theorem~\ref{thm:finite_main}}
We first state a stronger version of Lemma \ref{lem_I} and \ref{lem_target_err} in the following lemma (a) and (b).

\begin{lemma}[Stronger Version of Lemma~\ref{lem_I} and \ref{lem_target_err}]\label{lem_I_target_stronger}
	We assume that $L_{01}^S(g, g_{tc}) \le L_{01}^S(g^\ast, g_{tc}) + 2\Delta$. (In the previous proofs of Section~\ref{sec:main_theorem}, $\Delta = 0$.) Similarly, suppose $\Pbb_{\half(S+T)}[\widetilde{M}] \le C$, then we have the following results:
	
	(a) The ``inconsistency set'' $I$ is upper-bounded by
	$$\Pbb_S[\cup_{i\in I}S_i] \le \frac{4(C + \Delta)}{\gamma}.$$
	
	(b) The final target error is upper-bounded by
	$$\epsilon_T(g) \le \frac{8r(C+\Delta)}{\gamma}.$$
\end{lemma}

So we only need to find $C$ and $\Delta$. $C$ can be found by Lemma~\ref{lem_minority} by using a suitable $\hat{\mu}$ where $R_\Bcal(g) \le \hat{\mu}$. These results are given by the following lemma.

\begin{lemma}[Finite Sample Bound]\label{lem:finite_sample_bound}
	We have
	$$L_{01}^S(g, g_{tc}) \le L_{01}^S(g^\ast, g_{tc}) + 2\Delta$$
	and
	$$R_\Bcal(g) \le \hat{\mu}$$
	for
	$$\begin{aligned}
	\Delta  = \widetilde{O}& \left(\left(\Pbb_{x\sim\hat{S}}[m(f^\ast, x, g_{tc}(x))\le t] - L_{01}^{\hat{S}}(g^\ast, g_{tc})\right)\right. \\
	& + \left. \frac{\sum_i\sqrt{q}\|W_i\|_F}{t\sqrt{n}} + \sqrt{\frac{\log(1/\delta)+p\log n}{n}} \right),\\
	\hat{\mu} = \mu& + \widetilde{O}\rbr{\frac{\sum_i\sqrt{q}\|W_i\|_F}{t\sqrt{n}} + \sqrt{\frac{\log(1/\delta)+p\log n}{n}}}.\end{aligned}$$
\end{lemma}
And by plugging in the results from Lemma~\ref{lem_I_target_stronger} and \ref{lem:finite_sample_bound}, along with the constant $C$ in Lemma~\ref{lem_minority}, we immediately get the result in Theorem~\ref{thm:finite_main}. The proof of Lemma~\ref{lem_minority} can be founded in the proof of Lemma~\ref{lem_minority_2} in Appendix~\ref{sec:full_proof}, so we only need to prove Lemma~\ref{lem_I_target_stronger} and \ref{lem:finite_sample_bound} below.

\paragraph{Proof of Lemma~\ref{lem_I_target_stronger}.}

\begin{proof}
	
	(a). We only need to modify the proof of Lemma~\ref{lem_I_2} (Appendix~\ref{sec:proof_lem_2}), equation (\ref{equ:proof_lem_2}), in the following way (where $\kappa = 2$):
	
	$$\begin{aligned}
	&\;\;\;\;\;L_{01}^S(g, g_{tc}) \\\nonumber
	& = \sum_{i\in I}\Pbb_S[S_i]\Pbb_{x\sim S_i}[g(x) \neq g_{tc}(x)] \\
	& \;\;\;\; + \sum_{i\in \{1, \cdots, m\}\backslash I}\Pbb_S[S_i]\Pbb_{x\sim S_i}[g(x) \neq g_{tc}(x)] \\
	& \ge \sum_{i\in I}\Pbb_S[S_i]\rbr{\Pbb_{x\sim S_i}[g(x) \neq y_i] + \frac{\gamma}{2}} \\
	& \;\;\;\; + \sum_{i\in \{1, \cdots, m\}\backslash I}\Pbb_S[S_i]\sbr{\Pbb_{x\sim S_i}[g_{tc}(x)\neq y_i] - \Pbb_{S_i}[\widetilde{M_i}]}\\
	& \;\;\;\;\text{(by definition of $I$ for the first term, }\\
	& \;\;\;\;\;\; \text{by Lemma~\ref{lem_smallresult} for the second term)}\\
	& = L_{01}(g^\ast, g_{tc}) + \frac{\gamma}{2}\sum_{i\in I}\Pbb_S[S_i] \\
	& \;\;\;\; - \sum_{i\in \{1, \cdots, m\}\backslash I}\Pbb_S[\widetilde{M}_i]\\
	& \ge L_{01}^S(g, g_{tc}) - 2\Delta + \frac{\gamma}{2}\sum_{i\in I}\Pbb_S[S_i] - 2C,
	\end{aligned}$$
	
	and we immediately obtain 
	$$\sum_{i\in I}\Pbb_S[S_i] \le \frac{4(C+\Delta)}{\gamma}.$$
	
	(b). Similarly, we only need to modify the proof of Lemma~\ref{lem_target_err_2} (a) (Appendix~\ref{sec:proof_lem_3}), equation (\ref{equ:proof_lem_3}) based on the previous Lemma~\ref{lem_I_target_stronger} in the following way (where $\kappa = 2$):
	
	$$\begin{aligned}
	\sum_{i\in I}\epsilon_T^i(g) 
	& \le \sum_{i\in I} \Pbb_T[T_i] \\
	& \le \sum_{i\in I} r\Pbb_S[S_i] \\
	& \le \frac{4r(C+\Delta)}{\gamma}.
	\end{aligned}$$
	
	And since Lemma~\ref{lem_target_err_2} (b) 
	$$\sum_{i\in \{1, \cdots, K\}\backslash I}\epsilon_T^i(g) \le \frac{4rC}{\gamma}$$
	holds without change, together we easily have 
	$$\epsilon_T(g) \le \frac{8r(C+\Delta)}{\gamma}.$$
	
\end{proof}

\paragraph{Proof of Lemma~\ref{lem:finite_sample_bound}.}

\begin{proof}
	
	By Proposition~\ref{proposition1}, we know that: 
	\begin{align*}
	& \;\;\;\;\, L_{01}^S(g, g_{tc}) - L_{01}^S(g^\ast, g_{tc})\\
	& \le \Pbb_{x\sim \hat{S}}[m(f, x, g_{tc}(x))\le t] -  L_{01}^S(g^\ast, g_{tc})\\
	& \;\;\;\;+ \widetilde{O}\rbr{\frac{\sum_i\sqrt{q}\|W_i\|_F}{t\sqrt{n}} + \sqrt{\frac{\log(1/\delta)+p\log n}{n}}}\\
	& \le \Pbb_{x\sim \hat{S}}[m(f, x, g_{tc}(x))\le t] -  L_{01}^S(g^\ast, g_{tc})\\
	& \;\;\;\;+ \widetilde{O}\rbr{\frac{\sum_i\sqrt{q}\|W_i\|_F}{t\sqrt{n}} + \sqrt{\frac{\log(1/\delta)+p\log n}{n}}}\\
	& \text{(by algorithm (\ref{algorithm_margin}))} \\
	& \le \Pbb_{x\sim \hat{S}}[m(f, x, g_{tc}(x))\le t] -  L_{01}^{\hat{S}}(g^\ast, g_{tc})\\
	& \;\;\;\;+ O\rbr{\sqrt{\frac{\log(1/\delta)}{n}}}\\
	& \;\;\;\;+ \widetilde{O}\rbr{\frac{\sum_i\sqrt{q}\|W_i\|_F}{t\sqrt{n}} + \sqrt{\frac{\log(1/\delta)+p\log n}{n}}}\\
	& \text{(by standard concentration bound)} \\
	& = 2\Delta.
	\end{align*}
	
	By Proposition~\ref{proposition2}, we have 
	\begin{align*}
	R_{\Bcal}(g) & \le \Pbb_{x\sim \half(\hat{S} + \hat{T})}[m_{\Bcal}(f, x) \le t] \\
	& + \widetilde{O}\rbr{\frac{\sum_i\sqrt{q}\|W_i\|_F}{t\sqrt{n}} + \sqrt{\frac{\log(1/\delta)+p\log n}{n}}}\\
	& \le \hat{\mu}. \;\;\text{(by algorithm (\ref{algorithm_margin}))}
	\end{align*}
	
	And the lemma is proved.
	
\end{proof}

%% file: exp_detail.tex
\section{Detailed Experimental Settings}
\label{appsec:detail_exp}
In this section, we describe the detailed setting of our experiments.

\subsection{Dataset}
\paragraph{ENTITY-30~\citep{santurkar2021breeds}.} We use the ENTITY-30 dataset from BREEDS~\citep{santurkar2021breeds} to simulate natural subpopulation shift. ENTITY-30 is constructed by data from ImageNet. It consists of 30 superclasses of entities, e.g., insect, carnivore, and passerine, which are the labels of classification task. Each superclass has eight subclasses; for example, the superclass insect has fly, leafhopper, etc., as its subclasses. The dataset is constructed by splitting each superclass's subclasses into two random and disjoint sets and assigning one of them to the source and the other to the target domain. Each subclass has the same probability of being chosen into source and target and has the same number of samples. This ensures the source and target datasets are approximately balanced w.r.t. superclass. To simulate subpopulation shift scenarios, we construct an unsupervised domain adaptation task. We provide labels of superclasses on the source domain and only unlabeled data on the target domain. The goal is to achieve good population accuracy in the target domain. In the randomly generated ENTITY-30 dataset we used for experiments, there are 157487 labeled samples in the source domain and 150341 unlabeled data in the target domain.
\paragraph{Office-31~\citep{saenko2010adapting}.} Office-31 is a standard domain adaptation dataset of three diverse domains, Amazon from Amazon website, Webcam by web camera and DSLR by digital SLR camera with 4,652 images in 31 unbalanced classes.
\paragraph{Office-Home~\citep{venkateswara2017deep}.} Office-Home is a more complex dataset containing 15,500 images from four visually very different domains: Artistic images, Clip Art, Product images, and Real-world images.

\subsection{Adaptation of FixMatch for Subpopulation Shift}\label{appsec:adaptation_fixmatch}
We adapt the state-of-the-art semi-supervised learning method FixMatch~\citep{sohn2020fixmatch} to the subpopulation shift. Unlike semi-supervised learning, where the support sets of unlabeled data and labeled data are inherently the same, the support sets of different domains may disjoint a lot in subpopulation shift problems. To enable label propagation, we need a good feature map to enable label propagation on the \emph{feature space}. Such a feature map should be obtained without the need for labels on the target domain. Under these constraints, we hypothesize that the feature map learned by modern self-supervised learning algorithms helps. Concretely, we use the feature map learned by SwAV~\citep{caron2020unsupervised} which simultaneously \emph{clusters} the data while \emph{enforcing consistency} between cluster assignments produced for different augmentations of the same image. This representation has two merits; first, it encourages subpopulations with similar representations to cluster in the feature space; second, it enforces the augmented samples to be close in the feature space. We expect that subclasses from the same superclass will be assigned to the same cluster and thus overlap in the feature space.

Our adaptation of FixMatch has the following pipeline: 
\begin{itemize}
	\item Step 1: We first finetune a ResNet50 model with pre-trained SwAV representation\footnote{Since pretraining from scratch requires much computation resource, we simply take the officially released checkpoint from \url{https://github.com/facebookresearch/swav}. Note this representation is learned on \emph{unlabeled} ImageNet training set, a superset of ENTITY-30 training set. There is no leakage of label information on the target domain.} on the source domain; 
	\item Step 2: Then we use this model as the base classifier and further finetune it with the objective function of FixMatch, i.e., supervised loss on weak augmentations of source samples plus consistency regularization, which encourages the prediction of the classifier on strong augmentations of a sample to be same to the prediction on weak augmentations of the sample\footnote{Empirically, FixMatch also combines self-training techniques that take the hard label of the prediction on weak augmentations and soft label for strong augmentations. We also use Distribution Alignment extension mentioned in Section 2.5 in FixMatch paper~\citep{sohn2020fixmatch}.}.
\end{itemize}

\subsection{Hyperparameter Settings and Training Details}
\subsubsection{Subpopulation Shift Datasets}
We evaluate four methods, i.e., Training only on the Source Domain (we use TSD for acronym), FixMatch, DANN~\citep{ganin2016domain}, MDD~\citep{zhang2019bridging}. For Step 1 of FixMatch mentioned in Section~\ref{appsec:adaptation_fixmatch}, we simply take the model training only on the source domain. 

We hardly tune the hyperparameter from their default values from the released repos: \url{https://github.com/facebookresearch/swav} for the hyperparameters for finetuning from SwAV, \url{https://github.com/kekmodel/FixMatch-pytorch} for Fixmatch training, \url{https://github.com/thuml/Transfer-Learning-Library} for DANN and MDD.

We train all models for 30 epochs using SGD (FixMatch is finetuned from TSD for 30 epochs). Each configuration is evaluated with 3 different random seeds, and the mean and standard deviation are reported. We follow the configuration of \texttt{eval\_semisup.py} in \url{https://github.com/facebookresearch/swav/blob/master/eval_semisup.py} for TSD and FixMatch with some scaling of learning rate together with batch size and further take the hyperparameters of FixMatch from \url{https://github.com/kekmodel/FixMatch-pytorch}. For TSD: we use an initial learning rate of $0.4$ for the last linear layer and $0.02$ for other layers; we decay the learning rate by 10 at epoch 15 and 22; we train on 4 NVIDIA RTX 2080 Ti GPUs with 64 samples on each GPU. For FixMatch, we use an initial learning rate of $0.1$ for the last linear layer and $0.005$ for other layers; we use a cosine learning rate decay; we train on 4 NVIDIA RTX 2080 Ti GPUs with 16 labeled data from the source domain and $3\times 16$ (set the hyperparameter $\mu$ in FixMatch to 3, whose default value is 7, due to the limitation of GPU memory of RTX 2080 Ti) unlabeled data from target domain; we use Distribution Alignment~\citep{berthelot2019remixmatch} extension from the FixMatch paper~\citep{sohn2020fixmatch} (Section 2.5)\footnote{As mentioned in the FixMatch paper, this extension encourages the model predictions to have the same class distribution as the labeled set, and their results show that this extension is effective when the number of labeled data is small. We find this extension is also helpful in our subpopulation shift setting (improves the accuracy from $68.5\%$ to $72.6\%$). We hypothesize that this is because distribution alignment helps to learn a suitable representation on which the separation and expansion of subpopulation are well-satisfied so that label information can propagate.}; we select the parameter $\lambda_u$ of FixMatch (the coefficient of the consistency loss) from $\cbr{1, 10, 100}$ and use 10 for our experiments; the threshold hyperparameter $\tau$ is set to the default value 0.95. Since DANN and MDD are already algorithms for unsupervised domain adaptation, we directly use the default hyperparameters from \url{https://github.com/thuml/Transfer-Learning-Library} for DANN and MDD. We train each DANN and MDD model on a single NVIDIA RTX 2080 Ti GPU as the original code does not support multi-GPU training, and we just keep it.

In all experiments, we use Pytorch 1.7.1 with CUDA version 10.1. For all optimizers, we use SGD with the Nesterov momentum 0.9 and weight decay 5e-4. For TSD and FixMatch, we further use NVIDIA's apex library to enable mixed-precision training (with optimization level O1).

\subsubsection{Classic Unsupervised Domain Adaptation Datasets}
We train MDD models on Office-31 and Office-Home datasets following the configuration in \url{https://github.com/thuml/Transfer-Learning-Library}. Then we finetune the learned model using FixMatch (with Distribution Alignment extension) for 20 epochs. We do not tune any hyperparameters but directly use the same learning rate scale and learning rate scheduler as MDD, i.e., the batch size is 64, the initial learning rate is 0.008 for the last layers while is 0.0008 for the backbone feature extractor (ResNet50)\footnote{The default batch size and initial learning rate of MDD are 32 and 0.004, we simultaneously scale them by a factor of 2 for using parallel computation.}, the learning rate at step $i$ follows the schedule $\text{lr} = \text{initial lr} \times (1+0.0002 i)^{-0.75}$. The hyperparameters of FixMatch is set as $\mu = 3$ and $\lambda_u = 1$.

In all experiments, we use Pytorch 1.7.1 with CUDA version 10.1. For all optimizers, we use SGD with the Nesterov momentum 0.9 and weight decay 5e-4. For FixMatch finetuning, we further use NVIDIA's apex library to enable mixed-precision training (with optimization level O1).

%% file: other_related.tex
\section{Other Related Works}
\label{appsec:other_related}
There are many works designing algorithms based on the idea of distributional matching~\citep{adel2017unsupervised,becker2013non,pei2018multi,jhuo2012robust,hoffman2018cycada,zhao2019multi,long2017deep}. We refer the readers to \citet{zhang2019transfer,zhuang2020comprehensive,zhao2020review} for comprehensive surveys.

Domain generalization is a fundamental extension of domain adaptation; the distinction to domain adaptation is made precisely in, e.g.,~\citet{gulrajani2020search}. Most domain generalization methods aim to incorporate the invariances across all training datasets instead of only being invariant to a specific test domain~\citep{ghifary2015domain}. Different types of invariances are leveraged through algorithms like invariant risk minimization or its variants~\citep{arjovsky2019invariant,ahuja2020invariant,parascandolo2020learning,javed2020learning,krueger2020out,mitrovic2020representation}, (group) distributional robust optimization~\citep{sagawa2019distributionally}, and meta-learning algorithms~\citep{li2018learning}. Theoretical understandings on the invariant representation have also been stutied~\citep{zhao2020fundamental}. Other works also study how the inductive bias of models helps to generalize or extrapolate~\citep{xu2021how}.